\documentclass[journal,twocolumn]{IEEEtran}

\usepackage{graphicx}
\usepackage{amsmath}
\usepackage{amsfonts}
\usepackage{amssymb}
\usepackage[amssymb]{SIunits}
\usepackage{longtable}

\usepackage{import}
\usepackage{color}
\usepackage{algorithm,algpseudocode}
\usepackage{amsthm}
\usepackage{commath}
\usepackage{subcaption} 
\usepackage{lipsum}
\usepackage{multicol}
\usepackage{multirow}

\usepackage{cite}
\usepackage{pdfpages}

\newtheorem{mydef}{Definition}
\newtheorem{myremark}{Remark}
\newtheorem{mytheo}{Theorem}

\newtheorem{myprop}{Proposition}
\newtheorem{myproperty}{Property}
\newtheorem{myassump}{Assumption}

\ifCLASSINFOpdf
\else
\fi
\newsavebox\mybox

\begin{document}
%

\title{High-dimensional Neural Feature Design for Layer-wise Reduction of Training Cost}

%
%
%

%
%

\author{\IEEEauthorblockN{Alireza M. Javid, Arun Venkitaraman, Mikael Skoglund, and Saikat Chatterjee}\\
School of Electrical Engineering and Computer Science \\
KTH Royal Institute of Technology, Sweden}

%
%

\markboth{Submitted to EURASIP Journal on Advances in Signal Processing, 2020}%
{Shell \MakeLowercase{\textit{et al.}}: Bare Demo of IEEEtran.cls for IEEE Journals}
%



\maketitle

\begin{abstract}
We design a ReLU-based multilayer neural network by mapping the feature vectors to a higher dimensional space in every layer. We design the weight matrices in every layer to ensure a reduction of the training cost as the number of layers increases. Linear projection to the target in the higher dimensional space leads to a lower training cost if a convex cost is minimized. An $\ell_2$-norm convex constraint is used in the minimization to reduce the generalization error and avoid overfitting. The regularization hyperparameters of the network are derived analytically to guarantee a monotonic decrement of the training cost, and therefore, it eliminates the need for cross-validation to find the regularization hyperparameter in each layer. We show that the proposed architecture is norm-preserving and provides an invertible feature vector, and therefore, can be used to reduce the training cost of any other learning method which employs linear projection to estimate the target.
\end{abstract}


Nonlinear mapping of low-dimensional signal to high-dimensional space is a traditional method for constructing useful feature vectors, specifically for classification problems.
The intuition is that, by extending to a high dimension, the feature vectors of different classes become easily separable by a linear classifier. The drawback of performing classification in a higher-dimensional space is the increased computational complexity. This issue can be handled by a well-known method called 'kernel trick' in which the complexity depends only on the inner products in the high-dimensional space. Support vector machine (SVM) \cite{SVM_1995} and kernel PCA (KPCA) \cite{KPCA_1998} are examples of creating high-dimensional features by employing the kernel trick. The choice of the kernel function is a critical aspect that can affect the classification performance in the higher dimensional space. A popular kernel is the radial basis function (RBF) kernel or Gaussian kernel, and its good performance is justified by its ability to map the feature vector to a very high, infinite, dimensional space \cite{Bishop}. In this manuscript, we design a high-dimensional feature using an artificial neural network (ANN) architecture to achieve a better classification performance by increasing the number of layers. The architecture uses the rectified linear unit (ReLU) activation, predetermined orthonormal matrices, and a fixed structured matrix. We refer to this as High-dimensional neural feature (HNF) throughout the manuscript.

Neural networks and deep learning architectures have received overwhelming attention over the last decade \cite{DL_SPMag}. Appropriately trained neural networks have been shown to outperform the traditional methods in different applications, for example in classification and regression tasks\cite{Russakovsky2015,DodgeK17b}. By the continually increasing computational power, the field of machine learning is being enriched with active research pushing classification performance to higher levels for several challenging datasets \cite{pmlr-v28-wan13,GoodInit2016,pmlr-v51-lee16a}. However, very little is known regarding how many numbers of neurons and layers are required in a network to achieve better performance. Usually, some rule-of-thumb methods are used for determining the number of neurons and layers in an ANN, or an exhaustive search is employed which is extremely time-consuming \cite{NodeNum_2015}. In particular, the technical issue - guaranteeing performance improvement with increasing the number of layers - is not straight-forward in traditional neural network architectures, e.g., deep neural network (DNN) \cite{DNN_2013}, convolutional neural network (CNN) \cite{CNN_2012}, recurrent neural network (RNN) \cite{RNN_2013}, etc. We endeavor to address this technical issue by mapping the feature vectors to a higher dimensional space using predefined weight matrices.  

There exist several works employing predefined weight matrices that do not need to be learned. Scattering convolution network \cite{ScatteringNet_2013} is a famous example of these approaches which employs wavelets-based scattering transform to design the weight matrices. Random matrices have also been widely used as a mean for reducing the computational complexity of neural networks while achieving comparable performance as with fully-learned networks \cite{mathematicsDN,giryes_randomweights,SSFN_Saikat,PLN_Saikat}. In the case of the simple, yet effective, extreme learning machine (ELM), the first layer of the network is assigned randomly chosen weights and the learning takes place only at the end layer \cite{elm_Huang2012,elm_HUANG2015,elm_8085130,HELM}. It has also been shown recently that a similar performance to fully-learned networks may be achieved by training a network with most of the weights assigned randomly and only a small fraction of them being updated throughout the layers \cite{randomDN}. It has been shown that networks with Gaussian random weights provide a distance-preserving embedding of the input data \cite{giryes_randomweights}. The recent work \cite{PLN_Saikat} designs a deep neural network architecture called progressive learning network (PLN) which guarantees the reduction of the training cost with increasing the number of layers. In PLN, every layer is comprised of a predefined random part and a projection part which is trained individually using a convex cost function. These approaches indicate that randomness has much potential in terms of high-performance at low computational complexity. We design a multilayer neural network using predefined orthonormal matrices, e.g., random orthonormal matrix, DCT matrix, etc, to ensure reducing the training cost as the number of layers increases.

\subsection{Our contributions}
Motivated by the prior use of fixed matrices, we design the HNF architecture using an appropriate combination of ReLU, random matrices, and fixed matrices. We use predefined weight matrices in every layer of the network, and therefore, the architecture does not suffer from the infamous vanishing gradient problem. We theoretically show that the output of each layer provides a richer representation compared to the previous layers if a convex cost is minimized to estimate the target. We use an $\ell_2$-norm convex constraint to reduce the generalization error and avoid overfitting to the training data. We analytically derive the regularization hyperparameter to ensure the decrement of the training cost in each layer. Therefore, there is no need for cross-validation to find the optimum regularization hyperparameters of the network. We show that the proposed HNF is norm-preserving and invertible, and therefore, can be used to improve the performance of other learning methods that use linear projection to estimate the target. Finally, we show the classification performance of the proposed HNF against ELM and state-of-the-art results. Note that a preliminary version of this manuscript has been submitted to ICASSP 2020 recently.


\subsection{Notations}
We use the following notations unless otherwise noted: We use bold capital letters, e.g., $\mathbf{W}$, to denote matrices and bold lowercase letters, e.g., $\mathbf{x}$, to denote vectors. We use calligraphic letter $\mathcal{M}$ to denote a set and $\mathcal{M}^c$ to denote compliment set. The cardinality of a set $\mathcal{M}$ is denoted by $|\mathcal{M}|$. For a scalar $x \in \mathbb{R}$, let us denote its sign and magnitude as $s(x) \in \{-1,+1\}$ and $|x|$, respectively, and write $x = s(x) |x|$. For a vector $\mathbf{x}$, we define the sign vector $\mathbf{s}(\mathbf{x})$ and magnitude vector $|\mathbf{x}|$ by the element-wise operation. We define $\mathbf{g}(\cdot)$ as a non-linear function comprised of a stack of element-wise ReLU activation functions. A vector $\mathbf{x}$ has non-negative part $\mathbf{x}^{+}$ and non-positive part $\mathbf{x}^{-}$ such that $\mathbf{x} = \mathbf{x}^{+} + \mathbf{x}^{-}$ and $\mathbf{g}(\mathbf{x}) = \mathbf{x}^{+}$. We use $\| \cdot \|$ and $\| \cdot \|_F$ to denote $\ell_2$-norm and Frobenius norm, respectively. For example, it can be seen that $\|\mathbf{x}\|^2 = \|\mathbf{x}^{+} \|^2 + \|\mathbf{x}^{-} \|^2$.

\section{Proposed method}
\label{sec:Discrimination}
In this section, we illustrate the motivation to design a high-dimensional feature vector by using ReLU activation function. We analyze the behavior of a single layer ReLU network to the input perturbation noise and show that by mapping the feature vectors to a higher dimension, we can increase the discrimination power of the ReLU network.

For an ANN, we wish to have noise robustness and discriminative power. We characterize this in the following definition.

\begin{mydef}[Noise Robustness and Point Discrimination]
	\label{def:Robust_definition}
	Let $\mathbf{x}_1$ and $\mathbf{x}_2$ be two input vectors such that $\mathbf{x}_1 \neq \mathbf{x}_2$, and we have outputs of ANN $\tilde{\mathbf{t}}_1 = \mathbf{f}(\mathbf{x}_1)$ and $\tilde{\mathbf{t}}_2 = \mathbf{f}(\mathbf{x}_2)$. We can characterize a perturbation scenario with the perturbation noise $\Delta$ as $\mathbf{x}_2 = \mathbf{x}_1 + \Delta$. We wish that the proposed ANN holds the property
	\begin{eqnarray}
		c_1  \|  \mathbf{x}_1 - \mathbf{x}_2 \|^2 \leq \|  \mathbf{f}(\mathbf{x}_1) - \mathbf{f}(\mathbf{x}_2) \|^2  \leq c_2  \|  \mathbf{x}_1 - \mathbf{x}_2 \|^2,
	\end{eqnarray}
	where $0 < c_1 \leq 1$ and $c_1 \le c_2$.
\end{mydef}

\noindent Note that the lower bound provides point discrimination power and the upper bound provides noise robustness to the input.
\subsection{Layer Construction}
We first concentrate on one block of ANN -- this is called a layer in the neural network literature. The layer has an input vector vector $\mathbf{q} \in \mathbb{R}^{m\times 1}$ and the output vector $\mathbf{y} = \mathbf{g}(\mathbf{W}\mathbf{q})$. The dimension of $\mathbf{y}$ is the number of neurons in the layer. If we can guarantee that the layer of ANN provides noise robustness and point discrimination property then, the full ANN comprising of multiple layers connected sequentially can be guaranteed to hold robustness and discriminative properties.   
We need to construct $\mathbf{W}$ in such a manner that the layer has noise robustness and discriminative power according to the Definition \ref{def:Robust_definition}.

\subsection{ReLU Activation and A Limitation}
We first show three essential properties of ReLU function, required to develop our main results. We then discuss one possible limitation of the ReLU function and propose a remedy to circumvent the problem.
\begin{myproperty}[Scaling]
	\label{prop:Scaling_property}
	ReLU function has a scaling property. If $\mathbf{y} = \mathbf{g}(\mathbf{W}\mathbf{q})$, then $a\mathbf{y} = \mathbf{g}(\mathbf{W}(a\mathbf{q}))$ for a scalar $a \ge 0$.
\end{myproperty}
\begin{myproperty}[Sparsity]
	ReLU function provides sparse output vector $\mathbf{y}$ such that $\| \mathbf{y} \|_0 \leq \mathrm{dim}(\mathbf{y})$.
\end{myproperty}
\begin{myproperty}[Noise Robustness]
	\label{property:property2}
	Let us consider $\mathbf{z} = \mathbf{W} \mathbf{q}$. For two vectors $\mathbf{q}_1$ and $\mathbf{q}_2$, we define corresponding vectors $\mathbf{z}_1 = \mathbf{W} \mathbf{q}_1$ and $\mathbf{z}_2 = \mathbf{W} \mathbf{q}_2$, and output vectors $\mathbf{y}_1 = \mathbf{g}(\mathbf{z}_1) = \mathbf{g}(\mathbf{W}\mathbf{q}_1)$ and $\mathbf{y}_2 = \mathbf{g}(\mathbf{z}_2) = \mathbf{g}(\mathbf{W}\mathbf{q}_2)$. 
	Now, we have the following relation 
	\begin{eqnarray}
		0 \leq \| \mathbf{y}_1 - \mathbf{y}_2 \|^2 = \| \mathbf{g}(\mathbf{z}_1) - \mathbf{g}(\mathbf{z}_2) \|^2 \leq \| \mathbf{z}_1 - \mathbf{z}_2 \|^2.
	\end{eqnarray}
\end{myproperty}
\noindent The proof of Property~\ref{property:property2} is shown in Appendix \ref{Append:Property_ReLU}. The upper bound relation holds Lipschitz continuity that provides noise robustness. On the other hand, the lower bound being zero cannot maintain a minimum distance between two points $\mathbf{y}_1$ and $\mathbf{y}_2$. 
An example of extreme effect is that when $\mathbf{z}_1$ and $\mathbf{z}_2$ are non-positive vectors, we get $\| \mathbf{y}_1 - \mathbf{y}_2 \|^2 = 0$. 
This may limit the capacity of the ReLU function for achieving a good discriminative power. A reason for the limitation `lower bound being zero' is due to the structure of the input matrix $\mathbf{W}$. We build an appropriate structure for the input matrix to circumvent the limitation.

We now engineer a remedy for this limitation. Let us consider 
$
\bar{\mathbf{y}}=\mathbf{g}(\mathbf{V}\mathbf{z})= \mathbf{g}(\mathbf{V} \mathbf{W} \mathbf{q}),
$
where $\mathbf{z} = \mathbf{W} \mathbf{q} \in \mathbb{R}^{n}$ and $\mathbf{V}$ is a linear transform matrix.
For two vectors $\mathbf{q}_1$ and $\mathbf{q}_2$, we have corresponding vectors $\mathbf{z}_1 = \mathbf{W} \mathbf{q}_1$ and $\mathbf{z}_2 = \mathbf{W} \mathbf{q}_2$, and output vectors $\bar{\mathbf{y}}_1 = \mathbf{g}(\mathbf{V}\mathbf{z}_1)$ and $\bar{\mathbf{y}}_2 = \mathbf{g}(\mathbf{V}\mathbf{z}_2)$. Our interest is to show that there exists a predefined matrix $\mathbf{V}$  for which we have both noise robustness and discriminative power properties.

\begin{myprop}
	\label{proposition:proposition_with_Vn}
	Let us construct a $\mathbf{V}$ matrix as follows
	\begin{eqnarray}
		\mathbf{V} = \left[ 
		\begin{array}{c}
			\mathbf{I}_n \\
			- \mathbf{I}_n
		\end{array}
		\right] \triangleq \mathbf{V}_n.
		\label{eq:V_definition}
	\end{eqnarray}
	For the output vectors $\bar{\mathbf{y}}_1 = \mathbf{g}(\mathbf{V}_n\mathbf{z}_1) \in \mathbb{R}^{2n}$ and $\bar{\mathbf{y}}_2 = \mathbf{g}(\mathbf{V}_n\mathbf{z}_2) \in \mathbb{R}^{2n}$, we have $\| \bar{\mathbf{y}}_1 \|^2 = \| \mathbf{z}_1 \|^2 $ and $\| \bar{\mathbf{y}}_2 \|^2 = \| \mathbf{z}_2 \|^2 $ and 
	\begin{eqnarray}
		0 < \frac{1}{2} \| \mathbf{z}_1 - \mathbf{z}_2 \|^2 \leq \| \bar{\mathbf{y}}_1 - \bar{\mathbf{y}}_2 \|^2  \leq \| \mathbf{z}_1 - \mathbf{z}_2 \|^2.
		\label{eq:LBAndUB_1}
	\end{eqnarray}
\end{myprop}

\noindent The proof of the above proposition can be found in the Appendix \ref{Append:Prop_1}. Based on the above proposition, we can interpret the effect of noise passing through such layer. Let $\mathbf{z}_2 = \mathbf{z}_1 + \Delta \mathbf{z}$, where $\Delta \mathbf{z}$ is a small perturbation noise. Note that $\Delta \mathbf{z} = \mathbf{z}_1 - \mathbf{z}_2 = \mathbf{W}[\mathbf{q}_1 - \mathbf{q}_2] = \mathbf{W} \, \Delta \mathbf{q}$. To investigate effect of perturbation noise, we now state our main assumption.

\begin{myassump}
	\label{assump:sign_pattern_change_in_z}
	\textnormal{Given a small $\| \Delta \mathbf{z} \|^2$, the sign patterns of $\mathbf{z}_1$ and $\mathbf{z}_2$ does not differ significantly. On the other hand, for a large perturbation noise $\| \Delta \mathbf{z} \|^2$, the sign patterns of $\mathbf{z}_1$ and $\mathbf{z}_2$ vary significantly.}
\end{myassump}

\noindent The above assumption means that for a small $\| \Delta \mathbf{z} \|^2$, the set $\mathcal{M}(\mathbf{z}_1,\mathbf{z_2}) = \{ i | s(z_1(i)) = s(z_2(i)) \neq 0 \}$ is close to a full set and $\mathcal{M}^c(\mathbf{z}_1,\mathbf{z_2})$ is close to an empty set. On the other hand, for a large $\| \Delta \mathbf{z} \|^2$, the set $\mathcal{M}(\mathbf{z}_1,\mathbf{z_2}) = \{ i | s(z_1(i)) = s(z_2(i)) \neq 0 \}$ is close to an empty set and $\mathcal{M}^c(\mathbf{z}_1,\mathbf{z_2})$ is close to a full set. Considering Assumption \ref{assump:sign_pattern_change_in_z}, we can present the following remark regarding the effect of noise in the layer.

\begin{myremark}[Effect of perturbation noise]
	\label{remark:noise_effect_in_z}
	\textnormal{For a small perturbation noise $\| \Delta \mathbf{z} \|^2$, we have $\| \bar{\mathbf{y}}_1 - \bar{\mathbf{y}}_2 \|^2 \approx \| \mathbf{z}_1 - \mathbf{z}_2 \|^2$. On the other hand, a large perturbation noise is attenuated.} 
\end{myremark}

\noindent This follows from the proof of Proposition~\ref{proposition:proposition_with_Vn}, specifically equations \eqref{eq:z_distance} and \eqref{eq:y_bar_distance_a}. In fact, if $\mathcal{M}^c = \emptyset$ then $\| \bar{\mathbf{y}}_1 - \bar{\mathbf{y}}_2 \|^2 = \| \mathbf{z}_1 - \mathbf{z}_2 \|^2$. We interpret that a small perturbation noise passes through the single layer $\mathbf{g}(\mathbf{V}\mathbf{z})$ almost not attenuated. Let us construct an illustrative example. Assume that $\mathcal{M}^c(\mathbf{z}_1,\mathbf{z_2})$ is a full set and $\forall i \in \mathcal{M}^c, \,\, |z_1(i)| = |z_2(i)|$. In that case, $\| \bar{\mathbf{y}}_1 - \bar{\mathbf{y}}_2 \|^2 = 0.5 \| \mathbf{z}_1 - \mathbf{z}_2 \|^2$ and we can comment that the perturbation noise is attenuated.

\section{High-dimensional neural feature}
\label{sec:HNF}
In this section, we employ the proposed weight matrix in \eqref{eq:V_definition} to construct a multilayer ANN. We show that by designing the weight matrices in every layer, it is possible to construct a network that provides noise robustness and point discrimination according to Definition \ref{def:Robust_definition}.

Let us establish the relation between the input vector $\mathbf{q} \in \mathbb{R}^m$ and output vector $\bar{\mathbf{y}} \in \mathbb{R}^{2n}$. For two vectors $\mathbf{q}_1$ and $\mathbf{q}_2$, we have corresponding vectors $\mathbf{z}_1 = \mathbf{W} \mathbf{q}_1$ and $\mathbf{z}_2 = \mathbf{W} \mathbf{q}_2$, and output vectors $\bar{\mathbf{y}}_1 = \mathbf{g}(\mathbf{V}_n\mathbf{z}_1) = \mathbf{g}(\mathbf{V}_n \mathbf{W}\mathbf{q}_1) $ and $\bar{\mathbf{y}}_2 = \mathbf{g}(\mathbf{V}_n\mathbf{z}_2) = \mathbf{g}(\mathbf{V}_n \mathbf{W}\mathbf{q}_2)$. Our interest is to show that it is possible to construct a $\mathbf{W} \in \mathbb{R}^{n \times m}$ matrix for which we have both noise robustness and discriminative power properties. We can construct $\mathbf{W}  \in \mathbb{R}^{n \times m}$ as orthonormal matrix, such that $n \geq m$ and $\mathbf{W}^{\top}\mathbf{W}=\mathbf{I}_m$. In that case, we have $\| \mathbf{q}_1 - \mathbf{q}_2 \|^2 = \| \mathbf{z}_1 - \mathbf{z}_2 \|^2$ for any pair of $(\mathbf{q}_1,\mathbf{q}_2)$. By combining the this relation with the equation \eqref{eq:LBAndUB_1}, we conclude the following proposition.

\begin{myprop}
	\label{prop:Noise_robustness_discrimination_ability_for_SingleLayer}
	Consider the single layer network $\bar{\mathbf{y}} = \mathbf{g}(\mathbf{V}_n\mathbf{z}) = \mathbf{g}(\mathbf{V}_n \mathbf{W}\mathbf{q})$ where $\mathbf{W} \in \mathbb{R}^{n \times m}$ is
	an orthonormal matrix, such that $n \geq m$ and $\mathbf{W}^{\top}\mathbf{W}=\mathbf{I}_m$. Then, $\| \bar{\mathbf{y}} \|^2 = \|\mathbf{q} \|^2 $, and for every two vectors $\mathbf{q}_1$ and $\mathbf{q}_2$, the following inequality holds
	\begin{eqnarray}
		\frac{1}{2} \| \mathbf{q}_1 - \mathbf{q}_2 \|^2 \leq \| \bar{\mathbf{y}}_1 - \bar{\mathbf{y}}_2 \|^2 \leq  \| \mathbf{q}_1 - \mathbf{q}_2 \|^2.
	\end{eqnarray}
\end{myprop} 

\noindent The above proposition shows that by designing the weight matrix in a single layer network, it is possible to provide point discrimination and noise robustness according to Definition \ref{def:Robust_definition}. Note that the weight matrix $\mathbf{W}$ can be any orthonormal matrix such as instances of random orthonormal matrix, DCT matrix, etc. By considering the relation $\Delta \mathbf{z} = \mathbf{W} \, \Delta \mathbf{q}$, we can present a similar argument as in Remark~\ref{remark:noise_effect_in_z}. We interpret that a small perturbation noise $\| \Delta \mathbf{q} \|^2$ passes through the single layer $\mathbf{g}(\mathbf{V} \mathbf{W} \mathbf{q})$ almost not attenuated. This is stated in the following remark.
\begin{myremark}[Effect of perturbation noise]
	\label{remark:noise_effect_in_q}
	\textnormal{For a small $\| \Delta \mathbf{q} \|^2$, we have $\| \bar{\mathbf{y}}_1 - \bar{\mathbf{y}}_2 \|^2 \approx \| \mathbf{q}_1 - \mathbf{q}_2 \|^2$. On the other hand, a large perturbation noise is attenuated.}
\end{myremark}

\noindent  By directly using Proposition \ref{proposition:proposition_with_Vn}, we can present a similar bound in regards to the perturbation of the weight matrix $\mathbf{W}$ in a single layer construction. We can show that the perturbation norm in the output due to the perturbation to the weight matrix has an upper bound that is a scaled version of the input norm. 
The scaling parameter $\| \Delta \mathbf{W} \|_F^2$ is small for a small perturbation. The following remark illustrates this point in detail.
\begin{myremark}[Sensitivity to the weight matrix]
	\label{remark:system_perturbation}
	\textnormal{Let the weight matrix $\mathbf{W}$ be perturbed by $\Delta \mathbf{W}$. The effective weight matrix is $\mathbf{W} + \Delta \mathbf{W}$. For an input $\mathbf{q}$ and the respective outputs $\bar{\mathbf{y}}= \mathbf{g}(\mathbf{V}_n \mathbf{W} \mathbf{q})$ and $\bar{\mathbf{y}}_{\Delta}= \mathbf{g}(\mathbf{V}_n [\mathbf{W} + \Delta \mathbf{W}] \mathbf{q})$, we have
		\begin{align}
			\| \bar{\mathbf{y}} - \bar{\mathbf{y}}_{\Delta} \|^2 & \leq \| \Delta \mathbf{W} \|_F^2  \| \mathbf{q} \|^2.
	\end{align}}
\end{myremark}

\noindent The proof can be found in Appendix \ref{Append:Remark_3}. 

\subsection{Multilayer Construction}
A feedforward ANN is comprised of similar operational layers in a chain. Let us consider two layers in feedforward connection, e.g., $l$-th and $(l+1)$-th layers of an ANN. For the $l$-th layer, we use a superscript $(l)$ to denote appropriate variables and parameters. Let the $l$-th layer has $m^{(l)}$ nodes. The input to the $l$-th layer is $\mathbf{q}^{(l)} = \bar{\mathbf{y}}^{(l-1)}$. The output of $l$-th layer $\bar{\mathbf{y}}^{(l)} = \mathbf{g}(\mathbf{V}_{n^{(l)}}\mathbf{z}^{(l)}) = \mathbf{g}(\mathbf{V}_{n^{(l)}} \mathbf{W}^{(l)}\mathbf{q}^{(l)}) $ is next used as the input to the $(l+1)$-th layer, that means $\bar{\mathbf{y}}^{(l)} = \mathbf{q}^{(l+1)}$. Thus, the output of $(l+1)$-th layer is
\begin{align}
	\bar{\mathbf{y}}^{(l+1)} & = \mathbf{g}(\mathbf{V}_{n^{(l+1)}}\mathbf{z}^{(l+1)}) \nonumber \\ & = \mathbf{g}(\mathbf{V}_{n^{(l+1)}} \mathbf{W}^{(l+1)}\mathbf{q}^{(l+1)}) \nonumber \\  & = \mathbf{g}(\mathbf{V}_{n^{(l+1)}} \mathbf{W}^{(l+1)} \mathbf{g}(\mathbf{V}_{n^{(l)}} \mathbf{W}^{(l)}\mathbf{q}^{(l)}))
	\label{eq:Two_Blocks_IO_Relation}
\end{align}
Now, for the two vectors $\mathbf{q}^{(l)}_1$ and $\mathbf{q}^{(l)}_2$, we have the following relations in $l$-layer based on Proposition \ref{prop:Noise_robustness_discrimination_ability_for_SingleLayer}
\begin{eqnarray}
	\frac{1}{2} \| \mathbf{q}_1^{(l)} - \mathbf{q}_2^{(l)} \|^2 \leq \| \bar{\mathbf{y}}_1^{(l)} - \bar{\mathbf{y}}_2^{(l)} \|^2 \leq \| \mathbf{q}_1^{(l)} - \mathbf{q}_2^{(l)} \|^2.
\end{eqnarray}
We present the above results as the following theorem to provide noise robustness and discrimination power properties of the proposed ANN and call it High-dimensional Neural Feature (HNF) afterwards.

\begin{mytheo}
	\label{theo:RobustnessOfANN_forOthogonalMatrix}
	The proposed HNF uses ReLU activation function and is constructed as follows.
	\begin{enumerate}
		\item[(a)] The HNF is comprised of $L$ layers where the $l$-th layer has the corresponding structure $\bar{\mathbf{y}}^{(l)} = \mathbf{g}(\mathbf{V}_{n^{(l)}} \mathbf{W}^{(l)} \bar{\mathbf{y}}^{(l-1)})$. The $L$ layers are in a chain. The input to the first layer is $\mathbf{q}^{(1)} = \mathbf{x}$. The output of HNF is
		\begin{align*}
			\bar{\mathbf{y}}^{(L)} = \mathbf{g}(\mathbf{V}_{n^{(L)}} \mathbf{W}^{(L)} \mathbf{g}(\hdots \mathbf{g}(\mathbf{V}_{n^{(1)}} \mathbf{W}^{(1)} \mathbf{x}))).
		\end{align*}
		\item[(b)] In the HNF, $\mathbf{W}^{(l)} \in \mathbb{R}^{n^{(l)} \times m^{(l)}}$ matrices are orthonormal matrices with appropriate sizes, that is $n^{(l)} \ge m^{(l)}$ and $m^{(l)}=2n^{(l-1)}$.
	\end{enumerate}
	Then, $\| \bar{\mathbf{y}}^{(L)} \|^2 = \|\mathbf{x} \|^2 $, and the construted HNF provides noise robustness and discriminative power properties that are characterized by the following relation
	\begin{eqnarray}
		\frac{1}{2^L}\| \mathbf{x}_1 - \mathbf{x}_2 \|^2 \leq \| \bar{\mathbf{y}}_1^{(L)} - \bar{\mathbf{y}}_2^{(L)} \|^2 \leq  \| \mathbf{x}_1 - \mathbf{x}_2 \|^2,
	\end{eqnarray}
	where $\mathbf{x}_1 \in \mathbb{R}^{m^{(1)}}$ and $\mathbf{x}_2 \in \mathbb{R}^{m^{(1)}}$ are two input vectors to the HNF and their corresponding outputs are $\bar{\mathbf{y}}_1^{(L)}$ and $\bar{\mathbf{y}}_2^{(L)}$, respectively.
\end{mytheo}

\noindent Note that a similar argument as in Remark~\ref{remark:noise_effect_in_q} holds here as well. We interpret that a small perturbation noise passes through the multilayer structure almost not attenuated. On the other hand, a large perturbation noise is attenuated in every layer. Using Theorem \ref{theo:RobustnessOfANN_forOthogonalMatrix}, we follow similar arguments as in Remark \ref{remark:system_perturbation} in regard to the perturbation of the weight matrices $\mathbf{W}^{(l)}$ in every layer of the HNF
.
\begin{myremark}[Sensitivity to the weight matrix]
	\textnormal{Consider a scenario where the weight matrix $\mathbf{W}^{(l)}$ is perturbed by $\Delta\mathbf{W}^{(l)}$. The effective weight matrix is $\mathbf{W}^{(l)} + \Delta\mathbf{W}^{(l)}$. We can show that
		\begin{align}
			\| \bar{\mathbf{y}}^{(L)} - \bar{\mathbf{y}}_{\Delta}^{(L)} \|^2 & \leq \prod_{l=1}^L \| \Delta \mathbf{W}^{(L)} \|_F^2  \| \mathbf{x} \|^2.
	\end{align}}
\end{myremark}

\section{Reduction of training cost}
\label{sec:Improved_feature_vector_generation}
In this section, we analyze the effectiveness of the weight matrix $\mathbf{V}$ in the sense of reducing the training cost. We show that the proposed HNF provides lower training costs as the number of layers increases. We also present how the proposed structure can be used to reduce the training cost of other learning methods which employ linear projection to the target. 

Consider a dataset containing $N$ samples of pair-wise $P$-dimensional input data $\mathbf{x} \in \mathbb{R}^{P}$ and $Q$-dimensional target vector $\mathbf{t} \in \mathbb{R}^{Q}$ as $\mathcal{D}=\{(\mathbf{x},\mathbf{t})\}$.
Let us construct two single layer neural networks and compare effectiveness of their feature vectors. In one network, we construct the feature vector as $\mathbf{y} = \mathbf{g}(\mathbf{W} \mathbf{x})$, and in the other network, we build the feature vector $\bar{\mathbf{y}} = \mathbf{g}(\mathbf{V}_n \, \mathbf{W} \, \mathbf{x})$. We use the same input vector $\mathbf{x}$, predetermined weight matrix $\mathbf{W} \in \mathbb{R}^{n \times P}$, and ReLU activation function $\mathbf{g}(\cdot)$ for both networks. However, in the second network, the effective weight matrix is $\mathbf{V}_n \mathbf{W}$ where 
$
\mathbf{V}_n = \left[  
\begin{array}{c}
\mathbf{I}_n \\
- \mathbf{I}_n
\end{array}
\right]  \in \mathbb{R}^{2n \times n}
$
is fully pre-determined. To predict the target, we use a linear projection of feature vector. Let the predicted target for the first network be $\mathbf{O} \mathbf{y}$, and the predicted target for the second network $ \bar{\mathbf{O}} \bar{\mathbf{y}}$. Note that $\mathbf{O} \in \mathbb{R}^{Q \times n}$ and $\bar{\mathbf{O}} \in \mathbb{R}^{Q \times 2n}$. By using $\ell_2$-norm regularization, we find optimal solutions for the following convex optimization problems.
\begin{subequations}
	\begin{flalign}
		\mathbf{O}^{\star} & = \arg\displaystyle\min_{\mathbf{O}} E \left[ \| \mathbf{t} - \mathbf{O} \mathbf{y} \|^2 \right] \,\, \mathrm{s.t.} \,\, \|\mathbf{O} \|_F^2 \leq \epsilon, & \\
		\bar{\mathbf{O}}^{\star} & = \arg\displaystyle\min_{\bar{\mathbf{O}}} E \left[ \| \mathbf{t} - \bar{\mathbf{O}} \bar{\mathbf{y}} \|^2 \right] \,\, \mathrm{s.t.} \,\, \| \bar{\mathbf{O}} \|_F^2 \leq \epsilon, \label{eq:Optimization_VW} &
	\end{flalign}  
\end{subequations}
where the expectation operation is done by sample averaging over all $N$ data points in the training dataset. The regularization parameter $\epsilon$ is the same for the two networks. By defining $\mathbf{z} \triangleq \mathbf{W} \mathbf{x}$, we have
\begin{eqnarray}
	\bar{\mathbf{y}} = 
	\left[
	\begin{array}{c}
		\mathbf{z}^{+} \\ -\mathbf{z}^{-}
	\end{array}
	\right] = 
	\left[
	\begin{array}{c}
		\mathbf{y} \\ -\mathbf{z}^{-}
	\end{array}
	\right].
\end{eqnarray}
The above relation is due to the special structure of $\mathbf{V}_n$ and the use of ReLU activation $\mathbf{g}(\cdot)$. Note that the solution $\bar{\mathbf{O}}^{\star} = [\mathbf{O}^{\star} \,\, \mathbf{0}]$ exists in the feasible set of the minimization \eqref{eq:Optimization_VW}, i.e., $\| [\mathbf{O}^{\star} \,\, \mathbf{0}] \|_F^2 \le \epsilon$, where $\mathbf{0}$ is a zero matrix of size $Q \times n$. Therefore, we can show the optimal costs of the two networks have the following relation
\begin{eqnarray}
	E \left[ \| \mathbf{t} - \bar{\mathbf{O}}^{\star} \bar{\mathbf{y}} \|^2 \right]  \leq E \left[ \| \mathbf{t} - \mathbf{O}^{\star} \mathbf{y} \|^2 \right],
\end{eqnarray}
where the equality happens when $\bar{\mathbf{O}}^{\star} = [\mathbf{O}^{\star} \,\, \mathbf{0}]$. Any other optimal solution of $\bar{\mathbf{O}}$ will lead to inequality relation due to the convexity of the cost. Therefore, we can conclude that the feature vector $\bar{\mathbf{y}}$ of the second network is richer than the feature vector $\mathbf{y}$ of the first network in the sense of reduced training cost.    The proposed structure provides an additional property for the feature vector $\bar{\mathbf{y}}$ which we state in the following proposition. The proof idea of the proposition will be used in the next section to construct a multilayer structure, and therefore, we present the proof here.

\begin{myprop}
	For the feature vector $\bar{\mathbf{y}} = \mathbf{g}(\mathbf{V}_n \mathbf{W} \mathbf{x})$, there exists an invertible mapping $\mathbf{h}(\bar{\mathbf{y}}) = \mathbf{x}$ when the weight matrix $\mathbf{W}$ is full-column rank. 
	\label{prop:invertibility_single_layer}
\end{myprop}
\begin{proof}
	We now state Lossless Flow Property (LFP), as used in \cite{SSFN_Saikat,MI_for_NN_2018}. A non-linear function $\mathbf{g}(\cdot)$ holds the lossless flow property (LFP) if there exist two linear transformations $\mathbf{V}$ and $\mathbf{U}$ such that $\mathbf{U} \mathbf{g} (\mathbf{V} \mathbf{z}) = \mathbf{z}, \forall \mathbf{z} \in \mathbb{R}^n$. It is shown in \cite{SSFN_Saikat} that ReLU holds LFP. In other words, if 
	$
	\mathbf{V} \triangleq \mathbf{V}_n = \left[  
	\begin{array}{c}
	\mathbf{I}_n \\
	- \mathbf{I}_n
	\end{array}
	\right]  \in \mathbb{R}^{2n \times n}
	\,\, \mathrm{and} \,\, 
	\mathbf{U} \triangleq \mathbf{U}_n = \left[  
	\mathbf{I}_n  \,\, - \mathbf{I}_n
	\right] \in \mathbb{R}^{n \times 2n}
	$,
	then $\mathbf{U}_n \mathbf{g} (\mathbf{V}_n \mathbf{z}) = \mathbf{z}$ holds for every $\mathbf{z}$ when $\mathbf{g}(\cdot)$ is ReLU. Letting $\mathbf{z} = \mathbf{W}\mathbf{x}$, we can easily find  $\mathbf{x} = \mathbf{W}^{\dag} \mathbf{z} = \mathbf{W}^{\dag} \mathbf{U}_n \bar{\mathbf{y}}$, where $\dag$ denotes pseudo-inverse when $\mathbf{W}$ is a full-column rank matrix. Therefore, the resulting inverse mapping $\mathbf{h}$ would be linear. 
\end{proof}

\subsection{Reduction of Training Cost with Depth}
\label{subsec:Multilayer_feature_design}
In this section, we show that the proposed HNF provides lower training costs as the number of layers increases. Consider an $L$-layer feed-forward network according to our proposed structure on the weight matrices as follows
\begin{eqnarray}
	\bar{\mathbf{y}}^{(L)} = \mathbf{g}(\mathbf{V}_{n^{(L)}} \mathbf{W}^{(L)} \mathbf{g}(\hdots \mathbf{g}(\mathbf{V}_{n^{(1)}} \mathbf{W}^{(1)} \mathbf{x}))).
	\label{eq:L_layer_network}
\end{eqnarray}
Note that $2n^{(L)}$ is the number of neurons in the $L$-th layer of the network. The input-output relation in each layer is characterized by
\begin{subequations}
	\begin{align}
		\bar{\mathbf{y}}^{(1)} & = \mathbf{g}(\mathbf{V}_{n^{(1)}} \mathbf{W}^{(1)}\mathbf{x}), \label{eq:Initial_feature_vector}\\
		\bar{\mathbf{y}}^{(l)} &= \mathbf{g}(\mathbf{V}_{n^{(l)}} \mathbf{W}^{(l)}\bar{\mathbf{y}}^{(l-1)}), \,\,\, 2 \le l \le L,
	\end{align}
\end{subequations}
where $\mathbf{W}^{(1)} \in \mathbb{R}^{n^{(1)} \times P}$, $\mathbf{W}^{(l)} \in \mathbb{R}^{n^{(l)} \times m^{(l)}}$, and $m^{(l)}=2n^{(l-1)}$ for $2 \le l \le L$. Let the predicted target using the $l$-th layer feature vector $\bar{\mathbf{y}}^{(l)}$ be $\mathbf{O}_l \bar{\mathbf{y}}^{(l)}$. We find optimal solutions for the following convex optimization problems
\begin{subequations}
	\begin{flalign}
		\mathbf{O}_{l-1}^{\star} & = \arg\displaystyle\min_{\mathbf{O}} E \left[ \| \mathbf{t} - \mathbf{O} \bar{\mathbf{y}}^{(l-1)} \|^2 \right] \, \mathrm{s.t.} \, \|\mathbf{O} \|_F^2 \!\leq\! \epsilon_{l-1}, & \\
		\mathbf{O}_{l}^{\star} & = \arg\displaystyle\min_{\mathbf{O}} E \left[ \| \mathbf{t} - \mathbf{O} \bar{\mathbf{y}}^{(l)} \|^2 \right] \, \mathrm{s.t.} \, \|\mathbf{O} \|_F^2 \!\leq\! \epsilon_{l}. \label{eq:opt_lth_layer} &
	\end{flalign}
\end{subequations}
Let us define $\mathbf{z}^{(l)} \triangleq \mathbf{W}^{(l)}\mathbf{y}^{(l-1)}$. Assuming that weight matrices $\mathbf{W}^{(l)}$ are full-column rank, we can similarly derive $\mathbf{y}^{(l-1)} = [\mathbf{W}^{(l)}]^{\dag} \mathbf{z}^{(l)}$. By using Proposition \ref{prop:invertibility_single_layer}, we have $\mathbf{z}^{(l)} = \mathbf{U}_{n^{(l)}} \bar{\mathbf{y}}^{(l)}$ and then, we can write the following relations
\begin{align}
	\bar{\mathbf{y}}^{(l-1)} = [\mathbf{W}^{(l)}]^{\dag} \mathbf{z}^{(l)} = [\mathbf{W}^{(l)}]^{\dag} \mathbf{U}_{n^{(l)}} \bar{\mathbf{y}}^{(l)}, \label{eq:y_l_inverse}
\end{align}
where $\mathbf{U}_{n^{(l)}} = [\mathbf{I}_{n^{(l)}} \,\, -\mathbf{I}_{n^{(l)}}]$. If we choose $\mathbf{O}_{l}^{\star} = \mathbf{O}_{l-1}^{\star} [\mathbf{W}^{(l)}]^{\dag} \mathbf{U}_{n^{(l)}}$, by using \eqref{eq:y_l_inverse}, we can easily see that $\mathbf{O}_{l}^{\star} \bar{\mathbf{y}}^{(l)} = \mathbf{O}_{l-1}^{\star} \bar{\mathbf{y}}^{(l-1)}$. Therefore, by including $\mathbf{O}_{l-1}^{\star} [\mathbf{W}^{(l)}]^{\dag} \mathbf{U}_{n^{(l)}}$ in the feasible set of the minimization \eqref{eq:opt_lth_layer}, we can guarantee that the optimal cost of $l$-th layer would be lower or equal than that that of layer $(l-1)$. In particular, by choosing $\epsilon_{l} = \| \mathbf{O}_{l-1}^{\star} [\mathbf{W}^{(l)}]^{\dag} \mathbf{U}_{n^{(l)}} \|_F^2$, we can see that the optimal costs follow the relation
\begin{eqnarray}
	E \left[ \| \mathbf{t} - \mathbf{O}_{l}^{\star} \bar{\mathbf{y}}^{(l)} \|^2 \right] \leq E \left[ \| \mathbf{t} - \mathbf{O}_{l-1}^{\star} \bar{\mathbf{y}}^{(l-1)} \|^2 \right],
	\label{eq:relation_between_costs_layer_wise}
\end{eqnarray}
where the equality happens when we have $\mathbf{O}_{l}^{\star} = \mathbf{O}_{l-1}^{\star} [\mathbf{W}^{(l)}]^{\dag} \mathbf{U}_{n^{(l)}}$. Any other optimal solution of $\mathbf{O}_{l}$ will lead to inequality relation due to the convexity of the cost. Therefore, we can conclude that the feature vector $\bar{\mathbf{y}}^{(l)}$ of an $l$-layer network is richer than the feature vector $\bar{\mathbf{y}}^{(l-1)}$ of an $(l-1)$-layer network in the sense of reduced training cost. Note that if we choose the weight matrix $\mathbf{W}^{(l)}$ to be orthonormal, then
\begin{align}
	\epsilon_{l} &= \| \mathbf{O}_{l-1}^{\star} [\mathbf{W}^{(l)}]^{\top} \mathbf{U}_{n^{(l)}} \|_F^2 \nonumber\\
	& = \mathrm{trace}\left[ 2 \mathbf{I}_{n^{(l)}}  \mathbf{W}^{(l)} [\mathbf{O}_{l-1}^{\star}]^{\top} \mathbf{O}_{l-1}^{\star} [\mathbf{W}^{(l)}]^{\top}  \right] \nonumber\\
	& = 2 \, \mathrm{trace}\left[  \mathbf{W}^{(l)} [\mathbf{O}_{l-1}^{\star}]^{\top} \mathbf{O}_{l-1}^{\star} [\mathbf{W}^{(l)}]^{\top}  \right] \nonumber\\
	& = 2 \, \mathrm{trace}\left[  [\mathbf{W}^{(l)}]^{\top} \mathbf{W}^{(l)} [\mathbf{O}_{l-1}^{\star}]^{\top} \mathbf{O}_{l-1} ^{\star}  \right] \nonumber\\
	& = 2 \, \mathrm{trace}[ [\mathbf{O}_{l-1}^{\star}]^{\top} \mathbf{O}_{l-1}^{\star} ] = 2 \| \mathbf{O}_{l-1}^{\star} \|_F^2, \label{eq:epsilon_DCT}
\end{align}
where we have used the fact that $\mathbf{U}_{n^{(l)}} [\mathbf{U}_{n^{(l)}}]^{\top} = 2 \mathbf{I}_{n^{(l)}}$. As we have $\| \mathbf{O}_{l-1} \|_F^2 \leq \epsilon_{l-1}$, a sufficient condition to guarantee the cost relation \eqref{eq:relation_between_costs_layer_wise} is to use the relation between regularization parameters as $\epsilon_{l} \geq 2\epsilon_{l-1}$. We can choose $\epsilon_{l} = 2\epsilon_{l-1}=2^{l-1} \epsilon_1$. Note that the regularization parameter $\epsilon_1$ in the first layer can also be determined analytically. Consider $\mathbf{O}^{\star}_{ls}$ to be the solution of the following least-squares optimization
\begin{eqnarray}
	\mathbf{O}_{\text{ls}}^{\star} = \arg\displaystyle\min_{\mathbf{O}} E \left[ \| \mathbf{t} - \mathbf{O} \mathbf{x} \|^2 \right].
	\label{eq:LS_solution}
\end{eqnarray}
Note that the above minimization has a closed-form solution. Similar to the argument in \eqref{eq:relation_between_costs_layer_wise}, by choosing $\epsilon_1 = \| \mathbf{O}_{\text{ls}}^{\star} [\mathbf{W}^{(1)}]^{\dag} \mathbf{U}_{n^{(1)}} \|_F^2$, it can be easily seen that 
\begin{eqnarray}
	E \left[ \| \mathbf{t} - \mathbf{O}_{1}^{\star} \bar{\mathbf{y}}^{(1)} \|^2 \right] \le E \left[ \| \mathbf{t} - \mathbf{O}^{\star}_{\text{ls}} \mathbf{x} \|^2 \right],
	\label{eq:relation_cost_LS}
\end{eqnarray}
where the equality happens only when we have $\mathbf{O}_1^{\star} = \mathbf{O}_{\text{ls}}^{\star} [\mathbf{W}^{(1)}]^{\dag} \mathbf{U}_{n^{(1)}}$. Similar to Proposition \ref{prop:invertibility_single_layer}, we can prove the following proposition regarding the invertibility of the feature vector at the $l$-th layer of the proposed structure.
\begin{myprop}
	For the feature vector $\bar{\mathbf{y}}^{(L)}$ in \eqref{eq:L_layer_network}, there exists an invertible mapping function $\bar{\mathbf{h}}(\bar{\mathbf{y}}^{(L)}) = \mathbf{x}$ when the set of  weight matrices $\{ \mathbf{W}^{(l)} \}_{l=1}^L$ are full-column rank. 
\end{myprop}
\begin{proof}
	It can be proved by repeatedly using the lossless flow property (LFP) similar to Proposition \ref{prop:invertibility_single_layer}.
\end{proof}

\subsection{Reduction of Training Cost of ELM}
Note that the feature vector $\bar{\mathbf{y}}^{(1)}$ in \eqref{eq:Initial_feature_vector} can be any feature vector that is used for linear projection to the target in any other learning method. In Subsection \ref{subsec:Multilayer_feature_design}, we assume $\bar{\mathbf{y}}^{(1)}$ to be the feature vector constructed from $\mathbf{x}$ using the matrix $\mathbf{V}$; and therefore, the regularization parameter $\epsilon_1$ is derived to guarantee performance improvement compared to least-square method as shown in \eqref{eq:relation_cost_LS}. A potential extension would be to build the proposed HNF using the feature vector $\bar{\mathbf{y}}^{(1)}$ from other methods that employ linear projection to estimate the target. For example, the extreme learning machine (ELM) uses a linear projection of the nonlinear features vector to predict the target \cite{elm_Huang2012}. In the following, we build the proposed HNF by employing the feature vector used in ELM to improve the performance. 

\begin{table*}[t!]
	\centering
	\caption{Test classification accuracy of the proposed HNF for different datasets using random matrices}
	\label{table:Database_for_classification}
	\setlength{\tabcolsep}{5pt}
	\renewcommand{\arraystretch}{1.5}
	\begin{tabular}{ |c|c|c|c|c||c|c|c||c|c|c|c||c| } 
		\hline
		\multirow{2}{*}{Dataset} & size of & size of & Input & Number of & \multicolumn{3}{|c||}{Proposed HNF} & ELM & \multicolumn{3}{|c||}{Proposed HNF} & state-of-the-art \\ \cline{6-12}
		& training data & testing data & dimension $(P)$ & classes $(Q)$ & Accuracy & $n^{(1)}$ & $L$ & Accuracy & Accuracy & $n^{(1)}$ & $L$ & [reference] \\ \hline \hline
		Letter & 13333 & 6667 & 16 & 26 & 93.3 & $250$ & 5 & 88.3 & 94.6 & $1000$ & 3 & 95.8 \cite{Tang_HELM_2016} \\ 
		\hline
		Shuttle & 43500 & 14500 & 9 & 7 & 99.3 & $250$ & 5 & 99.0 & 99.6 & $1000$ & 3 & 99.9 \cite{elm_Huang2012} \\ 
		\hline
		MNIST & 60000 & 10000 & 784 & 10 & 97.1 & $1000$ & 5 & 96.9 & 97.7 & $4000$ & 3 & 99.7 \cite{Wan_DropConnect_2013} \\ 
		\hline
	\end{tabular}
\end{table*}

Similar to equation \eqref{eq:relation_between_costs_layer_wise}, we can show that it is possible to improve the feature vector of ELM in the sense of training cost by using the proposed HNF. Consider $\bar{\mathbf{y}}^{(1)} = \mathbf{g}(\mathbf{W}^{(1)} \mathbf{x})$, to be feature vector used in ELM for linear projection to the target. In the ELM framework, $\mathbf{W}^{(1)} \in \mathbb{R}^{n^{(1)} \times P}$ is an instance of normal distribution, not necessarily full-column rank, and $\mathbf{g}(\cdot)$ can be any activation function, not necessarily ReLU. The optimal mapping to the target in ELM is found by solving the following minimization problem.
\begin{eqnarray}
	\mathbf{O}_{\text{elm}}^{\star} = \arg\displaystyle\min_{\mathbf{O}} E \left[ \| \mathbf{t} - \mathbf{O} \bar{\mathbf{y}}^{(1)} \|^2 \right].
	\label{eq:ELM_solution}
\end{eqnarray}
Note that this minimization problem has a closed-form solution. We construct the feature vector in the second layer of the HNF as 
\begin{equation}
	\bar{\mathbf{y}}^{(2)} = \mathbf{g}(\mathbf{V}_{n^{(2)}}\mathbf{W}^{(2)} \bar{\mathbf{y}}^{(1)}),
	\label{eq:initial_feature_vector}
\end{equation}
where $\mathbf{W}^{(2)} \in \mathbb{R}^{n^{(2)} \times m^{(2)}}$ and $m^{(2)}=n^{(1)}$. The optimal mapping to the target by using this feature vector can be found by solving 
\begin{eqnarray}
	\mathbf{O}_{2}^{\star} = \arg\displaystyle\min_{\mathbf{O}} E \left[ \| \mathbf{t} - \mathbf{O} \bar{\mathbf{y}}^{(2)} \|^2 \right] \, \mathrm{s.t.} \, \|\mathbf{O} \|_F^2 \!\leq\! \epsilon_2,
\end{eqnarray}
where $\epsilon_2$ is the regularization parameter. By choosing $\epsilon_{2} = \| \mathbf{O}_{\text{elm}}^{\star} [\mathbf{W}^{(2)}]^{\dag} \mathbf{U}_{n^{(2)}} \|_F^2$, we can see that the optimal costs follow the relation
\begin{eqnarray}
	E \left[ \| \mathbf{t} - \mathbf{O}_2^{\star} \bar{\mathbf{y}}^{(2)} \|^2 \right] \leq E \left[ \| \mathbf{t} - \mathbf{O}_{\text{elm}}^{\star} \bar{\mathbf{y}}^{(1)} \|^2 \right],
	\label{eq:relation_cost_ELM}
\end{eqnarray}
where the equality happens when we have $\mathbf{O}_{2}^{\star} = \mathbf{O}_{\text{elm}}^{\star} [\mathbf{W}^{(2)}]^{\dag} \mathbf{U}_{n^{(2)}}$. Otherwise, the inequality has to follow.    Similarly, we can continue to add more layer to improve the performance. Specifically, for $l$-th layer of the HNF, we have $\bar{\mathbf{y}}^{(l)} = \mathbf{g}(\mathbf{V}_{n^{(l)}}\mathbf{W}^{(l)} \bar{\mathbf{y}}^{(l-1)})$, and we can show that
equation \eqref{eq:relation_between_costs_layer_wise} holds here as well when the set of matrices $\{ \mathbf{W}^{(l)} \}_{l=2}^L$ are full-column rank.

\subsection{Practical Considerations}
The dimension of feature vector $\bar{\mathbf{y}}^{(l)}$ increases as the number of layers increases. For a multi-layer feedforward network, if we use orthonormal matrix $\mathbf{W}^{(l)}$ for $l$-th layer, then each layer produces a feature vector that has at least twice the dimension of the input feature vector. At the $L$-th layer, we get the dimension $2^L$ times of the input data dimension. Note that $\bar{\mathbf{y}}= \mathbf{g}(\mathbf{V}_n \mathbf{z})$ is norm preserving by Proposition \ref{proposition:proposition_with_Vn}, that means $\| \bar{\mathbf{y}} \|^2 = \| \mathbf{z} \|^2$. Using this principle successively, the full network is also norm preserving, that means $\|\bar{\mathbf{y}}^{(L)}\|^2 = \| \mathbf{x} \|^2$. Therefore, as the layer number increases the amplitudes of scalars of the feature vector $\bar{\mathbf{y}}^{(L)}$ diminishes at the rate of $2^L$. We show that the proposed HNF does not require a large number of layers to improve the performance. This also answers the natural question that whether many layers are practically required for an ANN. Note that since the dimension of the feature vector $\bar{\mathbf{y}}^{(l)}$ is growing exponentially as $2^l$, the proposed HNF is not suitable for cases where the input dimension is too large. One way to circumvent this issue is to employ the kernel trick \cite{Bishop} by using the feature vector $\bar{\mathbf{y}}^{(l)}$. We will address this solution in future works.

\section{Results and discussion}
\label{sec:Experiments}
In this section, we carry out experiments to validate the performance improvement and observe the effect of using the matrix $\mathbf{V}$ in the architecture of an HNF. We report our results for three popular datasets in the literature as in Table \ref{table:Database_for_classification}. Note that we only choose the datasets where the input dimension is not very large due to the computational complexities. Letter dataset \cite{Letter_dataset} contains a 16-dimensional feature vector for each of the 26 English alphabets from A to Z. Shuttle dataset \cite{Shuttle_dataset} belongs to the STATLOG project and contains a 9-dimensional feature vector that deals with the positioning of radiators in the space shuttles. MNIST dataset \cite{MNIST_dataset} contains grey-scale $28 \times 28$-pixel images of hand-written digits. Note that in all three datasets, the target vector $\mathbf{t}$ is one-hot vector of dimension $Q$ (the number of classes). The optimization method used for solving the minimization problem \eqref{eq:opt_lth_layer} is the Alternating Direction Method of Multipliers (ADMM) \cite{ADMM_Boyd}. The number of iterations of ADMM is set to 100 in all the simulations.
    
We carry out two sets of experiments. First, we implement the proposed HNF with a fixed number of layers by using instances of random matrices for designing the weight matrix in every layer. In this setup, the weight matrix $\mathbf{W}^{(l)} \in \mathbb{R}^{n^{(l)} \times m^{(l)}}$ is an instance of Gaussian distribution with appropriate size $n^{(l)} \ge m^{(l)}$ and entries drawn independently from $\mathcal{N}(0,1)$ to ensure being full-column rank. Second, we construct the proposed HNF by using discrete cosine transform (DCT), as an example of full-column rank weight matrix, instead of random matrices. In this scenario, we may need to apply zero-padding before DCT to build the weight matrix $\mathbf{W}^{(l)}$ with appropriate dimension. The step size in the ADMM algorithm is set accordingly in each of these experiments. Finally, we compare the performance and computational complextiy of HNF and backpropagation over the same-size network.

\begin{figure*}[t!]
	\centering
	\makebox[\linewidth][c]{\includegraphics[width=220mm]{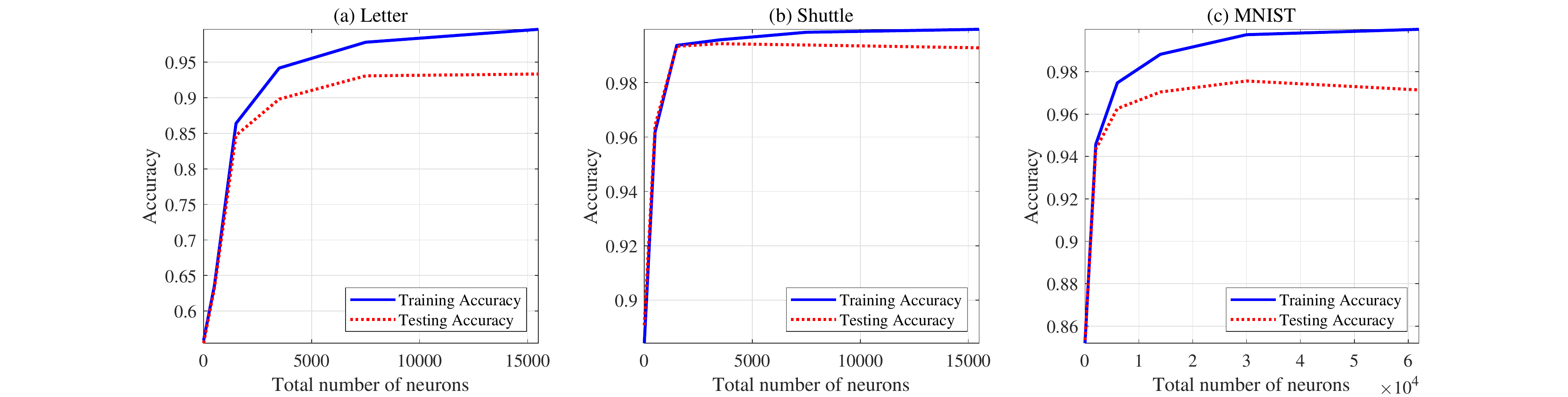}}
	\caption{Training and testing accuracy against size of one instance of HNF. Size of an $L$-layer HNF is represented by the number of random matrix based nodes, counted as $\sum_{l=1}^{L} 2 n^{(l)}$. Here, $L=5$ for all three datasets. The number of nodes in the first layer ($2n^{(1)}$) is set according to Table \ref{table:Database_for_classification} for each dataset.}
	\label{fig:accuracy_vs_size_ANN_LS}
\end{figure*}

\begin{figure*}[t!]
	\centering
	\makebox[\linewidth][c]{\includegraphics[width=220mm]{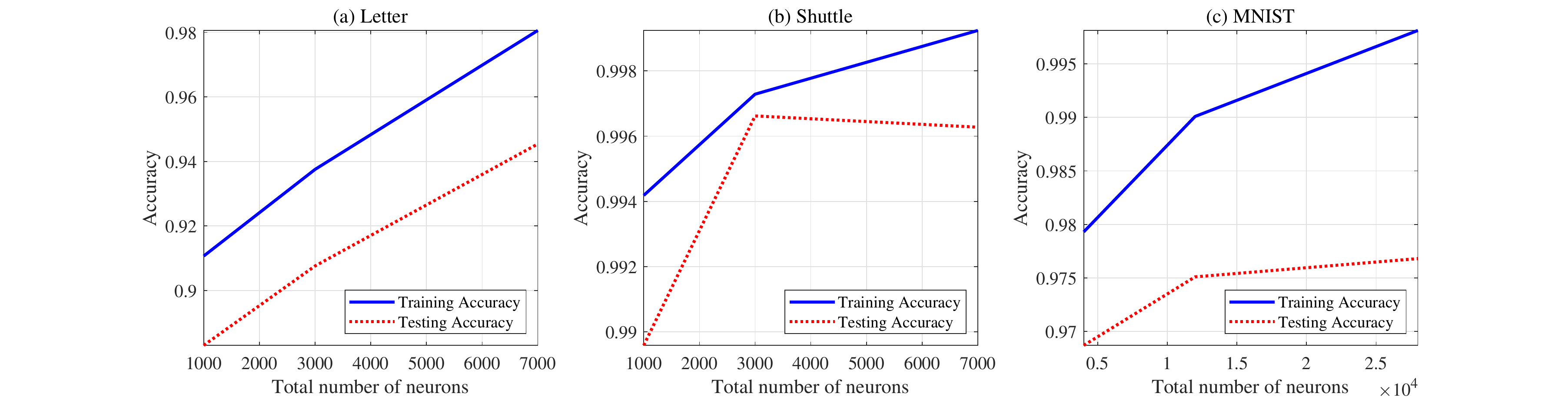}}
	\caption{Training and testing accuracy against size of one instance of HNF using ELM feature vector in the first layer. Size of an $L$-layer HNF is represented by the number of random matrix based nodes, counted as $n^{(1)} + \sum_{l=2}^{L} 2n^{(l)}$. Here, $L=3$ for all three datasets. The number of nodes in the first layer ($n^{(1)}$) is set according to Table \ref{table:Database_for_classification} for each dataset.}
	\label{fig:accuracy_vs_size_ANN_ELM}
\end{figure*}

\subsection{HNF Using Random Matrix}
\label{subsec:HNF_Random}
In this subsection, we construct the proposed HNF by using instances of Gaussian distribution to design the weight matrix $\mathbf{W}^{(l)}$. In particular, the entries of the weight matrix are drawn independently from $\mathcal{N}(0,1)$. For simplicity, the number of nodes is chosen according to $n^{(l)}=m^{(l)}$ for $l \ge 2$ in all the experiment. The number of nodes in the first layer $2n^{(1)}$ is chosen for each dataset individually such that it satisfies $n^{(1)} \ge P$ for every dataset with input dimension $P$, as reported in Table \ref{table:Database_for_classification}. The step size in the ADMM algorithm is set to $10^{-7}$ in all the simulations in this subsection. 
    
    We implement two different scenarios. First, we implement the proposed HNF with a fixed number of layers and show performance improvement throughout the layers. In this setup, the only hyperparameter that needs to be chosen is the number of nodes in the first layer $2n^{(1)}$. Note that the regularization parameter $\epsilon_1$ is chosen such that it guarantees \eqref{eq:relation_cost_LS}, and therefore eliminates the need for cross-validation in the first layer. Second, we build the proposed HNF by using the ELM feature vector in the first layer as in \eqref{eq:initial_feature_vector} and show the performance improvement throughout the layers. In this setup, the only hyperparameter that needs to be chosen is the number of nodes in the first layer $n^{(1)}$ which is the number of nodes of ELM to be exact. It has been shown that ELM performs better as the number of hidden neuron increases \cite{MI_for_NN_2018}, therefore, we choose a sufficiently large hidden neurons to make sure that ELM is performing at its best. Note that the regularization parameter $\epsilon_1$ is chosen such that it guarantees \eqref{eq:relation_cost_ELM}, and therefore, eliminates the need for cross-validation. Finally, we present the classification performance of the corresponding state-of-the-art results in Table \ref{table:Database_for_classification}.
    
    The performance results of the proposed HNF with $L=5$ layers are reported in Table \ref{table:Database_for_classification}. We report test classification accuracy as a measure to evaluate the performance. Note that the number of neurons $2n^{(1)}$ in the first layer of HNF is chosen appropriately for each dataset such that it satisfies $n^{(1)} \ge P$. For example, for MNIST dataset, we set $n^{(1)}=1000 \ge P=784$. The performance improvement in each layer of HNF is given in Figure \ref{fig:accuracy_vs_size_ANN_LS}, where train and test classification accuracy is shown versus total number of nodes in the network $\sum_{l=1}^{L} 2 n^{(l)}$. Note that the total number of nodes being zero corresponds to direct mapping of the input $\mathbf{x}$ to the target using least-squares according to \eqref{eq:LS_solution}. It can be seen that the proposed HNF provides a substantial improvement in performance with a small number of layers.
    
    The corresponding performance for the case of using the ELM feature vector in the first layer of HNF is reported in Table \ref{table:Database_for_classification}. It can be seen that HNF provides a tangible improvement in performance compared to ELM. Note that the number of neurons in the first layer $n^{(1)}$ is, in fact, the same as the number of neurons used in ELM. We choose $n^{(1)}$ to get the best performance for ELM in every dataset individually. The number of layers in the network is set to $L=3$ to avoid the increasing computational complexity. The performance improvement in each layer of HNF in this case is given in Figure \ref{fig:accuracy_vs_size_ANN_ELM}, where train and test classification accuracy is shown versus total number of nodes in the network $n^{(1)} + \sum_{l=2}^{L} 2n^{(l)}$. Note that the initial point corresponding to $n^{(1)}$ is in fact equal to the ELM performance reported in Table \ref{table:Database_for_classification}, which is derived according to \eqref{eq:ELM_solution}.
    
    Finally, we compare the performance of the proposed HNF with the state-of-the-art performance for these three datasets. We can see that the proposed HNF provides competitive performance compared to state-of-the-art results in the literature. It is worth mentioning that we have not used any pre-processing technique to improve the performance as in the the state-of-the-art, but it can be done in future works.

\begin{figure*}[t!]
	\centering
	\makebox[\linewidth][c]{\includegraphics[width=220mm]{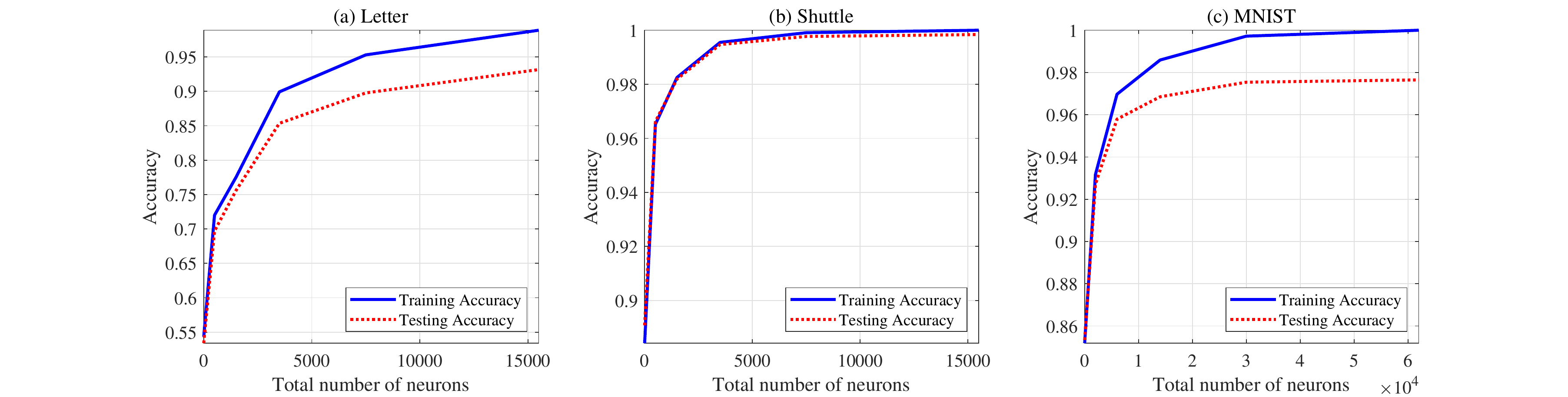}}
	\caption{Training and testing accuracy against size of HNF using DCT in every layer. Size of an $L$-layer HNF is represented by the number of DCT-based nodes, counted as $\sum_{l=1}^{L} 2 n^{(l)}$. Here, $L=5$ for all three datasets. The number of nodes in the first layer ($2n^{(1)}$) is set according to Table \ref{table:Database_for_classification_DCT} for each dataset.}
	\label{fig:accuracy_vs_size_ANN_LS_DCT}
\end{figure*}

\begin{figure*}[t!]
	\centering
	\makebox[\linewidth][c]{\includegraphics[width=220mm]{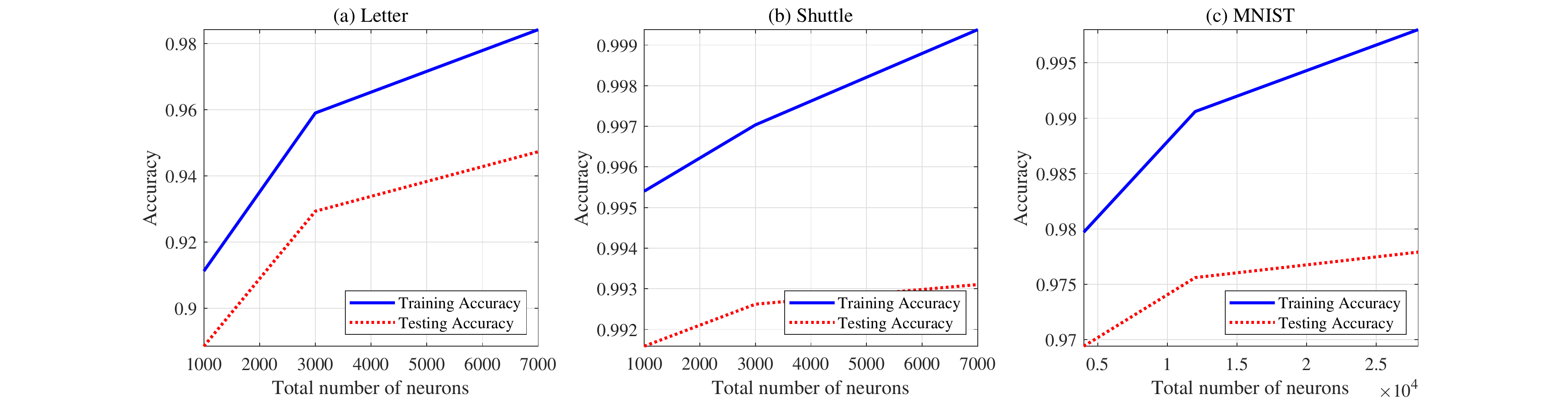}}
	\caption{Training and testing accuracy against size of HNF using ELM feature vector in the first layer and DCT in the next layers. Size of an $L$-layer HNF is represented by the number of nodes, counted as $n^{(1)} + \sum_{l=2}^{L} 2n^{(l)}$. Here, $L=3$ for all three datasets. The number of nodes in the first layer ($n^{(1)}$) is set according to Table \ref{table:Database_for_classification_DCT} for each dataset.}
	\label{fig:accuracy_vs_size_ANN_ELM_DCT}
\end{figure*}

\subsection{HNF Using DCT}
In this subsection, we repeat the same experiments as in Subsection \ref{subsec:HNF_Random} by using DCT instead of the Gaussian weight matrix. The number of nodes in each layer of the network is chosen as in Subsection \ref{subsec:HNF_Random}. We apply zero-padding before DCT in the first layer to build the weight matrix $\mathbf{W}^{(1)} \in \mathbb{R}^{n^{(1)} \times P}$ with appropriate dimension for each dataset. Note that $n^{(l)}=m^{(l)}$ for $l \ge 2$ in all the experiments, and therefore, there is no need to apply zero-padding in the next layers. The step size in the ADMM algorithm is set to $10^{2}$ in all the simulations in this subsection. 
    
    We implement the same two scenarios. First, we implement the proposed HNF by using DCT and show performance improvement throughout the layers. Second, we build the proposed HNF by using the ELM feature vector in the first layer and DCT matrices in the next layers. Note that the regularization parameters $\epsilon_l$ for $l \ge 2$ are chosen according to \eqref{eq:epsilon_DCT}. The choice of $\epsilon_1$ is such that it guarantees \eqref{eq:relation_cost_LS} and \eqref{eq:relation_cost_ELM} according to each scenario. 
    
\begin{table}[t!]
	\centering
	\caption{Test classification accuracy of the proposed HNF for different datasets using DCT}
	\label{table:Database_for_classification_DCT}
	\setlength{\tabcolsep}{5.2pt}
	\renewcommand{\arraystretch}{1.5}
	\begin{tabular}{|c|c|c|c||c|c|c|c|}
		\hline
		\multirow{2}{*}{Dataset} & \multicolumn{3}{|c||}{Proposed HNF} & ELM & \multicolumn{3}{|c|}{Proposed HNF} \\ \cline{2-8}
		& Accuracy & $n^{(1)}$ & $L$ & Accuracy & Accuracy & $n^{(1)}$ & $L$\\ \hline \hline
		Letter & 93.2 & $250$ & 5 & 88.3 & 94.7 & $1000$ & 3 \\ 
		\hline
		Shuttle & 99.8 & $250$ & 5 & 99.0 & 99.3 & $1000$ & 3 \\ 
		\hline
		MNIST & 97.7 & $1000$ & 5 & 96.9 & 97.8 & $4000$ & 3 \\ 
		\hline
	\end{tabular}
\end{table} 

    The performance results of the proposed HNF by using DCT matrices are reported in Table \ref{table:Database_for_classification_DCT}. Note that the number of neurons $n^{(1)}$ in the first layer and the number of layers are the same as Table \ref{table:Database_for_classification}. The performance improvement in each layer of HNF are given in Figure \ref{fig:accuracy_vs_size_ANN_LS_DCT} and Figure \ref{fig:accuracy_vs_size_ANN_ELM_DCT}. It can be seen that by using DCT in the proposed HNF, it is also possible to improve the performance with a few layers. 
    
    Finally, we compare the performance of the DCT-based HNF and that of the random matrix-based HNF as shown in Table \ref{table:Database_for_classification} and Table \ref{table:Database_for_classification_DCT}. We can see that using DCT as the weight matrix is as powerful as using random weights in these three datasets.

\begin{table}[t!]
        \centering
        \caption{Training time and test classification accuracy of the proposed HNF versus backpropagation}
        \label{table:time_complexity}
        \setlength{\tabcolsep}{3.6pt}
        \renewcommand{\arraystretch}{1.5}
        \begin{tabular}{|c|c|c||c|c|}
            \hline
            \multirow{2}{*}{Dataset} & \multicolumn{2}{|c||}{Proposed HNF} & \multicolumn{2}{|c|}{Backpropagation} \\ \cline{2-5}
            & Accuracy & Training time & Accuracy & Training time \\ \hline \hline
            Letter & $93.42$ & $47$ s & $95.03$ & $4566$ s \\ 
            \hline
            Shuttle & $99.21$ & $24$ s & $99.21$ & $15283$ s \\ 
            \hline
            MNIST & $97.14$ & $108$ s & $98.30$ & $20433$ s\\
            \hline
        \end{tabular}
    \end{table}

    \subsection{Computational Complexity}
    \label{subsec:HNF_backprop}
    Finally, we compare test classification accuracy and computational complexity of HNF with the backpropagation over the same learned HNF. We report training time of each method in seconds. We run our experiments on a server with multi-processors and 256 GB RAM. The optimization method used for backpropagation is ADAM  \cite{ADAM_2015} from TensorFlow. The learning rate of ADAM is chosen via cross-validation, and the number of epochs is fixed to 1000 in all the experiments.

    We construct HNF by using random weights and use the same number of layers and nodes as in Table \ref{table:Database_for_classification}. Note that we do not use ELM feature vector in the first layer for this experiments, although it is possible to use it in order to improve the performance. The results are shown in Table \ref{table:time_complexity}. As expected, backpropagation can improve the performance, except for Shuttle, at the cost of a significantly higher computational complexity. HNF, on the other hand, does not require cross-validation and only performs training at the last layer of the network, leading to a much faster training. Note that training time reported for backpropapation in Table \ref{table:time_complexity} does not include cross-validation for the learning rate so that we can have a fair comparison with HNF.

    At this point, we also provide the reported classification performance of scattering network on MNIST dataset for the sake of completeness. Scattering network with principal component analysis (PCA)  \cite{ScatteringNet_2013} over a modulus of windowed Fouriers transforms yields $98.2\%$ test classification accuracy for a spatial support equal to $8$. This results shows that scaterring network can outperform HNF at the cost of a higher complexity of using several scattering integrals in each layer. Note that  HNF only uses a random choice of a Gaussian distribution as the weight matrix in each layer. Besides, scattering network requires accurate choice of several hyperparameters such as the spatial support, number of filter banks, type of the transforms, and etc., which can be crucial for the performance. For example, in our experiments, a scattering network with PCA over a modulus of 2-D Morlet wavelets provides $94\%$ accuracy, at best, for a spatial support of $28$. The training on the our server lasted $1158$ seconds to yield such an accuracy, which highlights the learning speed of HNF in Table \ref{table:time_complexity}. The same network with a spatial support of $14$ gives a performance of $56.03\%$, showing the importance of a precise cross-validation.
    
\section{Conclusion}
\label{sec:Conclusion}
We show that by using a combination of orthonormal matrices and ReLU activation functions, it is possible to guarantee a monotonically decreasing training cost as the number of layers increases. The proposed method can be used by employing any other loss function, such as cross-entropy loss, as long as a linear projection is used after the ReLU activation function. Note that the same principle applies if instead of random matrices, we use any other real orthonormal matrices. Discrete cosine transform (DCT), Haar transform, and Walsh-Hadamard transform are examples of this kind. The proposed HNF is a universal architecture in the sense that it can be applied to improve the performance of any other learning method which employs linear projection to predict the target. The norm-preserving and invertibility of the architecture make the proposed HNF suitable for other applications such as auto-encoder design.

%

\appendix
\section{Appendix}

\subsection{Proof of Property~\ref{property:property2}}
\label{Append:Property_ReLU}
\begin{proof}
	For scalars $x_1$ and $x_2$, we have $y_1 = g(x_1)$ and  $y_2 = g(x_2)$. We have following relation
	\begin{eqnarray}
		(y_1 - y_2)^2 = \left \{
		\begin{array}{lr}
			(x_1 - x_2)^2 & \mathrm{if} \,\, x_1 > 0, x_2 >0  \\
			x_1^2 & \mathrm{if} \,\, x_1 > 0, x_2 <0  \\
			x_2^2 & \mathrm{if} \,\, x_1 < 0, x_2 > 0  \\
			0 & \mathrm{if} \,\, x_1 < 0, x_2 < 0.
		\end{array}
		\right.
	\end{eqnarray}
	Therefore, we find that ReLU function holds
	$
	0 \leq (y_1 - y_2)^2 \leq (x_1 - x_2)^2.
	$
	Considering the vectors $\mathbf{y}_1 = \mathbf{g}(\mathbf{z}_1) = \mathbf{g}(\mathbf{W}\mathbf{q}_1)$ and $\mathbf{y}_2 = \mathbf{g}(\mathbf{z}_2) = \mathbf{g}(\mathbf{W}\mathbf{q}_2)$, we have 
	\begin{align}
		0 & \leq \| \mathbf{y}_1 - \mathbf{y}_2 \|^2 = \sum_i (y_1(i) - y_2(i))^2 \nonumber\\ & \leq  \sum_i (z_1(i) - z_2(i))^2 = \| \mathbf{z}_1 - \mathbf{z}_2 \|^2,
	\end{align}
	where $y_1(i)$ is the the $i$-th scalar element of $\mathbf{y}_1$ and $z_1(i)$ is the the $i$-th scalar element of $\mathbf{z}_1$.
\end{proof}

\subsection{Proof of Proposition~\ref{proposition:proposition_with_Vn}}
\label{Append:Prop_1}
\begin{proof}
	We have $\mathbf{z} = \mathbf{W} \mathbf{q} \in \mathbb{R}^{n}$ and $\bar{\mathbf{y}}=\mathbf{g}(\mathbf{V}_n\mathbf{z}) \in \mathbb{R}^{2n}$ where 
	$
	\mathbf{V}_n = \left[ 
	\begin{array}{c}
	\mathbf{I}_n \\
	- \mathbf{I}_n
	\end{array}
	\right].
	$
	For two vectors $\mathbf{q}_1$ and $\mathbf{q}_2$, we have corresponding vectors $\mathbf{z}_1 = \mathbf{W} \mathbf{q}_1$ and $\mathbf{z}_2 = \mathbf{W} \mathbf{q}_2$, and output vectors $\bar{\mathbf{y}}_1 = \mathbf{g}(\mathbf{V}_n\mathbf{z}_1)$ and $\bar{\mathbf{y}}_2 = \mathbf{g}(\mathbf{V}_n\mathbf{z}_2)$. Note that $\bar{\mathbf{y}}_1 = 
	\left[
	\begin{array}{c}
	\mathbf{z}^{+}_1 \\ -\mathbf{z}^{-}_1
	\end{array}
	\right]$ and therefore, $\|\bar{\mathbf{y}}_1\|^2 = \| \mathbf{z}^{+}_1\|^2 + \| \mathbf{z}^{-}_1\|^2 = \| \mathbf{z}_1\|^2$, by definition. Similarly, $\|\bar{\mathbf{y}}_2\|^2 = \| \mathbf{z}_2\|^2$. Let us define a set 
	\begin{eqnarray}
		\mathcal{M}(\mathbf{z}_1,\mathbf{z_2}) \!=\! \{ i | s(z_1(i)) = s(z_2(i)) \neq 0 \} \subseteq \{1,2, \hdots , n\}. \nonumber
	\end{eqnarray}
	Then, we have
	\begin{align}
		\| \mathbf{z}_1 - \mathbf{z}_2 \|^2 & =  \sum_{i=1} ( z_1(i) - z_2(i) )^2 \nonumber \\
		& = \sum_{i} (s(z_1(i)) |z_1(i)| - s(z_2(i)) |z_2(i)|)^2 \nonumber \\
		& = \sum_{i \in \mathcal{M}(\mathbf{z}_1,\mathbf{z}_2)} ( |z_1(i)| - |z_2(i)| )^2 \nonumber \\ & + \sum_{i \in \mathcal{M}^c(\mathbf{z}_1,\mathbf{z}_2)} ( |z_1(i)| + |z_2(i)| )^2. 
		\label{eq:z_distance}
	\end{align}
	We write $\mathbf{z}_1 = \mathbf{z}_1^{+} + \mathbf{z}_1^{-} = \mathbf{s}(\mathbf{z}_1^{+}) |\mathbf{z}_1^{+}| + \mathbf{s}(\mathbf{z}_1^{-}) |\mathbf{z}_1^{-}|$. Then, after ReLU operation, we have
	$
	\bar{\mathbf{y}}_1 = \mathbf{g}(\mathbf{V}_n \mathbf{z}_1) = \left[
	\begin{array}{c}
	|\mathbf{z}_1^{+}|  \\
	|\mathbf{z}_1^{-}| 
	\end{array}
	\right]$ and $
	\bar{\mathbf{y}}_2 = \mathbf{g}(\mathbf{V}_n \mathbf{z}_2) = \left[
	\begin{array}{c}
	|\mathbf{z}_2^{+}|  \\
	|\mathbf{z}_2^{-}| 
	\end{array}
	\right].
	$
	\begin{table*}[t!]
		\begin{subequations}
			\begin{align}
				\| \bar{\mathbf{y}}_1 - \bar{\mathbf{y}}_2 \|^2 & = \| |\mathbf{z}_1^{+}| - |\mathbf{z}_2^{+}| \|^2 + \| |\mathbf{z}_1^{-}| - |\mathbf{z}_2^{-}| ) \|^2  \nonumber\\
				& = \displaystyle \sum_{i \in \mathcal{M}(|\mathbf{z}_1^{+}| ,|\mathbf{z}_2^{+}| )} ( |z_1^{+}(i)| - |z_2^{+}(i)| )^2 + \displaystyle \sum_{i \in \mathcal{M}^c(|\mathbf{z}_1^{+}| ,|\mathbf{z}_2^{+}| )} ( |z_1^{+}(i)| + |z_2^{+}(i)| )^2 \nonumber\\
				& + \displaystyle \sum_{i \in \mathcal{M}(|\mathbf{z}_1^{-}| ,|\mathbf{z}_2^{-}| )} ( |z_1^{-}(i)| - |z_2^{-}(i)| )^2 + \displaystyle \sum_{i \in \mathcal{M}^c(|\mathbf{z}_1^{-}| ,|\mathbf{z}_2^{-}| )} ( |z_1^{-}(i)| + |z_2^{-}(i)| )^2 \nonumber\\
				& = \displaystyle \sum_{i \in \mathcal{M}(|\mathbf{z}_1^{+}| ,|\mathbf{z}_2^{+}| )} ( |z_1^{+}(i)| - |z_2^{+}(i)| )^2 + \displaystyle \sum_{i \in \mathcal{M}(|\mathbf{z}_1^{-}| ,|\mathbf{z}_2^{-}| )} ( |z_1^{-}(i)| - |z_2^{-}(i)| )^2 \nonumber\\
				& + \displaystyle \sum_{i \in \mathcal{M}^c(|\mathbf{z}_1^{+}| ,|\mathbf{z}_2^{+}| )} ( |z_1^{+}(i)| + |z_2^{+}(i)| )^2 + \displaystyle \sum_{i \in \mathcal{M}^c(|\mathbf{z}_1^{-}| ,|\mathbf{z}_2^{-}| )} ( |z_1^{-}(i)| + |z_2^{-}(i)| )^2 \nonumber\\
				& = \displaystyle \sum_{i \in \mathcal{M}(\mathbf{z}_1,\mathbf{z}_2)} ( |z_1(i)| - |z_2(i)| )^2 + \displaystyle \sum_{i \in \mathcal{M}^c(\mathbf{z}_1,\mathbf{z}_2)}  |z_1(i)|^2 + |z_2(i)|^2. \nonumber\\
				& = \|\mathbf{z}_1 \|^2 + \| \mathbf{z}_2 \|^2 - 2 \displaystyle \sum_{i \in \mathcal{M}(\mathbf{z}_1,\mathbf{z}_2)} \mathbf{z}_1(i) \mathbf{z}_2(i) \nonumber\\
				& = \|\mathbf{z}_1 \|^2 + \| \mathbf{z}_2 \|^2 - 2 \displaystyle \sum_{i=1}^n \mathbf{z}_1(i) \mathbf{z}_2(i) + 2 \displaystyle \sum_{i \in \mathcal{M}^c(\mathbf{z}_1,\mathbf{z}_2)} \mathbf{z}_1(i) \mathbf{z}_2(i) \nonumber\\
				& = \|\mathbf{z}_1 - \mathbf{z}_2\|^2 + 2 \displaystyle \sum_{i \in \mathcal{M}^c(\mathbf{z}_1,\mathbf{z}_2)} \mathbf{z}_1(i) \mathbf{z}_2(i) \label{eq:y_bar_distance_a}\\ 
				& = \frac{1}{2}\|\mathbf{z}_1 - \mathbf{z}_2\|^2 + \frac{1}{2}(\|\mathbf{z}_1 \|^2 + \| \mathbf{z}_2 \|^2 - 2 \displaystyle \sum_{i \in \mathcal{M}(\mathbf{z}_1,\mathbf{z}_2)} \mathbf{z}_1(i) \mathbf{z}_2(i) + 2 \displaystyle \sum_{i \in \mathcal{M}^c(\mathbf{z}_1,\mathbf{z}_2)} \mathbf{z}_1(i) \mathbf{z}_2(i)) \nonumber\\
				& = \frac{1}{2}\|\mathbf{z}_1 - \mathbf{z}_2\|^2 + \frac{1}{2}(\|\mathbf{z}_1 \|^2 + \| \mathbf{z}_2 \|^2 - 2 \displaystyle \sum_{i \in \mathcal{M}(\mathbf{z}_1,\mathbf{z}_2)} |\mathbf{z}_1(i)| |\mathbf{z}_2(i)| - 2 \displaystyle \sum_{i \in \mathcal{M}^c(\mathbf{z}_1,\mathbf{z}_2)} |\mathbf{z}_1(i)| |\mathbf{z}_2(i)|) \nonumber\\
				& = \frac{1}{2}\|\mathbf{z}_1 - \mathbf{z}_2\|^2 + \frac{1}{2}(\|\mathbf{z}_1 \|^2 + \| \mathbf{z}_2 \|^2 - 2 \displaystyle \sum_{i = 1}^n |\mathbf{z}_1(i)| |\mathbf{z}_2(i)|) \nonumber\\
				& = \frac{1}{2}\|\mathbf{z}_1 - \mathbf{z}_2\|^2 + \frac{1}{2}\| |\mathbf{z}_1| - |\mathbf{z}_2| \|^2 \label{eq:y_bar_distance_b}
			\end{align}
		\end{subequations}
	\end{table*}
	With similar calculations as in \eqref{eq:z_distance}, we can derive the relationships in equation \eqref{eq:y_bar_distance_a} and \eqref{eq:y_bar_distance_b}. Since the summation $\sum_{i \in \mathcal{M}^c(\mathbf{z}_1,\mathbf{z}_2)} \mathbf{z}_1(i) \mathbf{z}_2(i)$ is always non-positive, from \eqref{eq:y_bar_distance_a}, we can see that
	\begin{eqnarray}
		\| \bar{\mathbf{y}}_1 - \bar{\mathbf{y}}_2 \|^2 \leq \| \mathbf{z}_1 - \mathbf{z}_2 \|^2,
	\end{eqnarray}
	where equality holds when $\mathcal{M}^c = \emptyset$, that means when sign patterns of $\mathbf{z}_1$ and $\mathbf{z}_2$ match exactly. From \eqref{eq:y_bar_distance_b}, it can also be seen that 
	\begin{eqnarray}
		\frac{1}{2}\| \mathbf{z}_1 - \mathbf{z}_2 \|^2 \le \| \bar{\mathbf{y}}_1 - \bar{\mathbf{y}}_2 \|^2,
	\end{eqnarray}
	where equality holds when $|\mathbf{z}_1| = |\mathbf{z}_2|$.
\end{proof}

\subsection{Proof of Remark \ref{remark:system_perturbation}}
\label{Append:Remark_3}
\begin{proof}
	Consider $\mathbf{z} = \mathbf{W} \mathbf{q}$ and $\mathbf{z}_{\Delta} \triangleq [\mathbf{W} + \Delta \mathbf{W}] \mathbf{q} = \mathbf{W} \mathbf{q} + [\Delta \mathbf{W}] \mathbf{q} = \mathbf{z} + \Delta \mathbf{z}$. Based on Proposition  \ref{proposition:proposition_with_Vn}, we can simply write
	\begin{align}
		\| \bar{\mathbf{y}} - \bar{\mathbf{y}}_{\Delta} \|^2 & \leq \| \Delta \mathbf{z} \|^2 = \| [\Delta \mathbf{W}] \, \mathbf{q} \|^2 \nonumber\\
		& \leq \| \Delta \mathbf{W} \|_F^2  \| \mathbf{q} \|^2,
	\end{align}
	where we have used equation \eqref{eq:LBAndUB_1}.
\end{proof}



%
%

\section*{Availability of data and materials}
All datasets used in the experiments are publicly available online. Please contact the corresponding author for simulation results.

\section*{Acknowledgements}
We acknowledge the support of our KTH colleagues Amirreza Zamani and Hamid Ghourchian for proofreading and critical remarks.

\ifCLASSOPTIONcaptionsoff
  \newpage
\fi



%

\bibliographystyle{IEEEbib}
\bibliography{HNF}
%

%


%



\end{document}